%% file: main_short.tex
\documentclass[runningheads]{llncs}

\usepackage[T1]{fontenc}
\usepackage{amsmath,amssymb}
\usepackage{graphicx}
\usepackage{booktabs}
\usepackage{xcolor}
\usepackage{microtype}
\usepackage[section]{placeins}
\input{macros}

\title{Complexity Bounds and Approaches to Learning Projected Gradient Descent Solver Iterates}
\titlerunning{Complexity Bounds for Learning Solver Iterates}

\author{Anjian Li\inst{1} \and Ryne Beeson\inst{2}}
\authorrunning{A. Li and R. Beeson}
\institute{Department of Electrical and Computer Engineering \\
\and Department of Mechanical and Aerospace Engineering \\ Princeton University, Princeton, NJ 08544, USA\\
\email{anjianl@princeton.edu} and \email{ryne@princeton.edu}}

\begin{document}
\maketitle

\begin{abstract}
Data scarcity poses a fundamental challenge in training generative models to produce initial guesses for parametric optimization problems that are otherwise numerically expensive to solve. 
We therefore study a $k$-neighborhood data collection strategy that augments datasets of converged solutions with intermediate solver iterates, increasing the amount of training data without additional solver runs.
To understand the benefits of this approach, we derive a generalization bound based on Rademacher complexity that reveals the role of the $k$-neighborhoods and related parameters. 
To achieve this result, we focus on one-sided box-constrained quadratic programs solved by projected gradient descent.
We illustrate the behavior of this solver on two examples. 
The approach proposed in this paper enables a more capable DDDAS paradigm by improving the efficiency of the data-model-optimization loop. 
We finish by discussing two views of learning solver-iterate data and connect our analysis with GLENS, a new data-efficient global search method.

\end{abstract}

\section{Introduction}\label{sec:introduction}

Data-driven global search aims to find a large collection of high-quality initial guesses for a family of parametric optimization problems.
A numerical solver can then refine these initial guesses toward local optima in a few iterations.
Generative models are useful in this setting because they can represent a distribution over multiple locally optimal solutions rather than a single deterministic solution.
Once trained, they can quickly generate diverse initial guesses for a new problem instance.
Prior work has used variational auto-encoders~\cite{li2023amortized,beeson2024global} and diffusion models~\cite{li2025diffusolve,li2024constraint,li2025aligning,Graebner:2025.jas.72.6,Graebner:2026.arxiv.jass} for this purpose in non-convex trajectory optimization.

The main limitation is the cost of collecting good training data.
Existing approaches generally keep only the final converged solution from each offline solver run.
Thus, producing a large dataset requires solving many problem instances with various initial guesses.
However, every solver run also produces a path of intermediate iterates, which is normally discarded.
These iterates are suboptimal but contain information about the local neighborhood of the optimum and require no additional cost to collect.
We therefore seek to leverage solver iterates to reduce the data requirement of generative learning.

Solver iterates may be viewed as points along a path toward an optimum or as one path embedded in a higher-dimensional space.
Because iterates from the same solver run are correlated, the analysis must account for their dependence, local spread, and convergence behavior.

This paper makes three contributions.
First, we formalize a definition of $k$-neighborhood data collection, in which the last $k+1$ iterates from each solver run are retained.
Second, for box-constrained quadratic programs (QPs), we show that projected gradient descent (PGD) contracts toward the optimum and use this behavior to derive a generalization bound with Rademacher complexity.
We illustrate this behavior when the optimum is interior or an upper bound is active.
Third, we discuss how these results inform methods that learn paths of solver data and their connection to the data-efficient global search method GLENS~\cite{li2026glens}.
From a DDDAS perspective, the solver supplies dynamically evolving data that is incorporated into a generative model that not only learns the behavior of the parameterized optimization problems and the dynamics of the numerical solver, but also provides accelerated future optimization and decision making.

The paper is organized as follows.
Section~\ref{sec:background} introduces the problem setting and background.
Section~\ref{sec:solver_iterates} defines the $k$-neighborhood and establishes the PGD contraction behavior.
Section~\ref{sec:generalization} develops the generalization bound with Rademacher complexity for box-constrained QPs, with examples shown in Section~\ref{sec:examples}.
Section~\ref{sec:discussion} discusses learning from solver iterates and GLENS. 
Section~\ref{sec:conclusion} concludes the paper.

\section{Problem Setup and Background}\label{sec:background}

\subsection{Parametric Optimization and Projected Gradient Descent}

We consider a family of parametric continuous optimization problems with the continuous decision variable $\vx\in\R^d$ and problem parameter $\vtheta\in\R^c$,
\begin{equation}\label{eq:parametric_problem}
\underset{\vx\in\R^d}{\mathrm{minimize}}\ f(\vx;\vtheta)
\quad\mathrm{subject\ to}\quad g(\vx;\vtheta)\le 0,\quad h(\vx;\vtheta)=0. \nonumber
\end{equation}
The objective function $f$ and constraint functions $g$ and $h$ have fixed forms across the problem family, while their values depend on $\vtheta$.

To solve this optimization problem, a numerical solver starts from an initial guess $\vx^0$ and produces a sequence of iterates $\vx^q$ that converges to a nearby locally optimal solution.
In this paper, we consider the one-sided box-constrained QP:
\begin{equation}\label{eq:box_qp}
\min_{\vx\in\R^d}\ f(\vx;\vtheta):=\frac{1}{2}\vx^T\mQ\vx-\vb^T\vx
\quad\text{subject to}\quad \vx\le\vu,
\end{equation}
where $\mQ\in\R^{d\times d} \succeq 0$, $\vb\in\R^d$, and $\vu\in\R^d$ denotes the upper bounds.
The problem parameter is $\vtheta=\operatorname{Flatten}(\mQ,\vb,\vu)\in\R^{d^2+2d}$.
The feasible set $\setC=\{\vx\in\R^d\mid\vx\le\vu\}$ is closed and convex, so problem~\eqref{eq:box_qp} has a unique minimizer $\vx^*(\vtheta)$.

For problem~\eqref{eq:box_qp}, PGD uses the single-step operator $T_{\vtheta}(\vx):=\ProjC{\vx-\stepsize\nabla f(\vx;\vtheta)}$,
where $\stepsize>0$ is the step size and $P_{\setC}$ is the Euclidean projection onto $\setC$.
Starting from an initial guess $\vx^0\in\setC$, PGD generates the sequence $\vx^q=T_{\vtheta}^q(\vx^0)$ for $q=0,1,\ldots$.
Let $\epsilon_{\mathrm{tol}}>0$ be a prescribed convergence tolerance.
We say that the solver converges after $n$ steps when the residual satisfies $\norm{T_{\vtheta}^{n-1}(\vx^0)-T_{\vtheta}^{n}(\vx^0)}_2\le\epsilon_{\mathrm{tol}}$.

\subsection{Generative Learning Framework}

The existing generative learning framework~\cite{li2023amortized,beeson2024global} uses a conditional generative model to learn the distribution of solutions $\vx$ given $\vtheta$.
Offline data collection produces converged solution--parameter pairs $\{(\vx_i^*,\vtheta_i)\}_{i=1}^N$, where $\vx_i^*:=\vx^*(\vtheta_i)$ is from solver run $i$.
Then, the generative model is trained to sample initial guesses that the solver refines to locally optimal solutions for a new problem parameter.

In this paper, we use a denoising diffusion probabilistic model (DDPM)~\cite{ho2020denoising} as an example because it has shown effectiveness in data-driven trajectory optimization~\cite{li2025diffusolve,li2024constraint}.
Let $t\in\{1,\ldots,T\}$ denote the diffusion step, let $(\alpha_t)_{t=1}^T$ with $\alpha_t\in(0,1)$ be a noise schedule, and define $\bar\alpha_t=\prod_{q=1}^t\alpha_q$.
At diffusion step $t$, the clean solution $\vx$ is perturbed as $\sqrt{\bar\alpha_t}\vx+\sqrt{1-\bar\alpha_t}\vepsilon$, where $\vepsilon\sim\mathcal{N}(0,I_d)$ is Gaussian noise.
We define the model input $\vz\in\R^{d+c+e}$ as $\vz=(\sqrt{\bar\alpha_t}\vx+\sqrt{1-\bar\alpha_t}\vepsilon,\vtheta,\vtau_t)$,
where $\vtau_t\in\R^e$ is a bounded embedding of the diffusion step.
A vector-valued noise predictor $g_{\vphi}:\R^{d+c+e}\to\R^d$, parameterized by $\vphi$, is trained with the loss
\begin{equation}\label{eq:diffusion_loss}
\ell(\vx,\vtheta,\vepsilon,t;\vphi)
=\norm{g_{\vphi}(\vz)-\vepsilon}_2.
\end{equation}
We assume that $g_{\vphi}$ belongs to the following two-layer neural network (NN) class:
\begin{equation}\label{eq:model_class}
\setG=\left\{\vz\mapsto\sum_{s=1}^{m}\va_s \psi(\vw_s^T\vz):
\sum_s\norm{\va_s}_2\le A,\ \norm{\vw_s}_2\le W\right\},
\end{equation}
where $m$ is the number of hidden units, $\va_s\in\R^d$ and $\vw_s\in\R^{d+c+e}$ are network weights with bounds $A>0$ and $W>0$.
The activation function $\psi:\R\to[0,1]$ satisfies $\psi(0)=0$ and is Lipschitz with constant $L_1>0$.
Hence $\norm{g_{\vphi}(\vz)}_2\le A$.

\section{Analysis of Solver-Iterate Data}\label{sec:solver_iterates}


We define a $k$-neighborhood to describe the solver iterates near a final converged solution, using the solver operator $T_{\vtheta}$ and convergence tolerance $\epsilon_{\mathrm{tol}}$ from Sec.~\ref{sec:background}.

\begin{definition}[$k$-neighborhood, radius, and path]\label{def:k_neighborhood}\label{def:k_neighborhood_radius}\label{def:k_neighborhood_path}
For a problem parameter $\vtheta$, a solver operator $T_{\vtheta}$, and $k \in \mathbb{N}_{+}$, the $k$-neighborhood is defined as the set of feasible initial points that converge after at most $k$ applications of $T_{\vtheta}$:
\begin{equation}\label{eq:k_neighborhood}
\mathcal{N}_{T_{\vtheta}}^k:=\left\{\vx\in\setC:\norm{T_{\vtheta}^{k-1}(\vx)-T_{\vtheta}^{k}(\vx)}_2\le\epsilon_{\mathrm{tol}}\right\}.
\end{equation}

We quantify its size by the radius $r_k(\vtheta)$, which is the maximum distance of the iterates to the corresponding optimum $\vx^*(\vtheta)$ and is assumed to be finite:
\begin{equation}\label{eq:radius}
r_k(\vtheta):=\sup_{\vx\in\mathcal{N}_{T_{\vtheta}}^k}\norm{\vx-\vx^*(\vtheta)}_2.
\end{equation}

Suppose solver run $i$ first satisfies the convergence criterion after $n_i\ge k$ steps.
The $k$-neighborhood path $\mathcal{S}^{\mathrm{solver}}_{k,i}$ is defined as the ordered sequence containing the final converged solution $T_{\vtheta_i}^{n_i}(\vx_i^0)$ and the preceding $k$ intermediate iterates:
\begin{equation}\label{eq:path}
\mathcal{S}^{\mathrm{solver}}_{k,i}
:=\bigl((T_{\vtheta_i}^{n_i-j}(\vx_i^0),\vtheta_i)\bigr)_{j=0}^{k}.
\end{equation}
\end{definition}



We next characterize the contraction behavior of PGD for box-constrained QPs.
Let $\lambda_{\max} \geq \lambda_{\min}>0$ be the largest and smallest eigenvalues of $\mQ$, and choose the step size as
\begin{equation}\label{eq:stepsize}
\stepsize=\frac{2}{\lambda_{\max}+\lambda_{\min}}.
\end{equation}

\begin{theorem}[Contractive behavior]\label{thm:contraction}
The PGD iterates for problem~\eqref{eq:box_qp} contract toward the optimum:
\begin{equation}\label{eq:contract}
\norm{T_{\vtheta}(\vx)-\vx^*(\vtheta)}_2
\le\rho\norm{\vx-\vx^*(\vtheta)}_2,
\qquad
\rho=\frac{\lambda_{\max}-\lambda_{\min}}{\lambda_{\max}+\lambda_{\min}}<1.
\end{equation}
\end{theorem}

\begin{proof}
The optimum satisfies the PGD fixed-point condition
$\vx^*(\vtheta)=\ProjC{\vx^*(\vtheta)-\stepsize\nabla f(\vx^*(\vtheta);\vtheta)}$.
By the non-expansiveness of projection onto a closed convex set,
\begin{align}
\norm{T_{\vtheta}(\vx)-\vx^*(\vtheta)}_2
&\le\norm{(\vx-\vx^*(\vtheta))-\stepsize(\nabla f(\vx;\vtheta)-\nabla f(\vx^*(\vtheta);\vtheta))}_2\\
&=\norm{(I_d-\stepsize\mQ)(\vx-\vx^*(\vtheta))}_2
\le\norm{I_d-\stepsize\mQ}_2\norm{\vx-\vx^*(\vtheta)}_2.
\end{align}
\end{proof}

With~\eqref{eq:stepsize}, $\norm{I_d-\stepsize\mQ}_2=\rho$.
Let $\rho_i$ denote the contraction factor for problem instance $i$.
Consider $\mathcal{S}^{\mathrm{solver}}_{k,i}$, whose earliest collected intermediate iterate is $T_{\vtheta_i}^{n_i-k}(\vx_i^0)$.
Since this point belongs to $\mathcal{N}_{T_{\vtheta_i}}^k$, Definition~\ref{def:k_neighborhood_radius} gives $\norm{T_{\vtheta_i}^{n_i-k}(\vx_i^0)-\vx^*(\vtheta_i)}_2\le r_k(\vtheta_i)$.
Applying the contraction bound for the next $k-j$ steps yields
\begin{equation}\label{eq:k_path_contraction}
\norm{T_{\vtheta_i}^{n_i-j}(\vx_i^0)-\vx^*(\vtheta_i)}_2
\le \rho_i^{k-j}r_k(\vtheta_i),
\qquad j=0,\ldots,k.
\end{equation}
Thus, all points in $\mathcal{S}^{\mathrm{solver}}_{k,i}$ lie in regions around the optimum bounded by $r_k(\vtheta_i)$.

\section{Generalization Bounds Based on Rademacher Complexity}\label{sec:generalization}\label{sec:pgd_complexity}

Generalization bounds quantify the difference between training and expected performance.
PAC-Bayes theory evaluates the Kullback--Leibler divergence between a prior and posterior distribution over predictors (e.g., the NN~\cite{mcallester1998some}).
Conformal prediction provides uncertainty quantification after training with finite-sample guarantees and without any assumptions about the distribution~\cite{angelopoulos2023conformal}.
However, our goal is to analyze how solver-iterate data affects the complexity of learning.
Rademacher complexity measures how well a function class can fit random signs on the observed sample~\cite{mohri2018foundations}.
It allows the input radius induced by the solver to appear in the bound and is therefore the most natural choice for our analysis.

We first present a standard Rademacher complexity result, and then extend it to $k$-neighborhood paths.
Let $s=(\vx,\vtheta,\vepsilon,t) \sim \mathcal{D}$ denote one training sample, and define $\ell_{\vphi}(s):=\ell(\vx,\vtheta,\vepsilon,t;\vphi)$ using~\eqref{eq:diffusion_loss}.
The corresponding loss function class is
$\setL:=\{s\mapsto\ell_{\vphi}(s):g_{\vphi}\in\setG\}$,
where $\setG$ is the NN class in~\eqref{eq:model_class}.
For $N$ independent samples $S=\{s_i\}_{i=1}^N$ drawn from $\mathcal{D}$, we define the empirical and true risks as
\begin{equation}\label{eq:risk_definitions}
\hatE(\vphi):=\frac{1}{N}\sum_{i=1}^N\ell_{\vphi}(s_i),
\qquad
\Etrue(\vphi):=\mathbb{E}_{s\sim\mathcal{D}}[\ell_{\vphi}(s)].
\end{equation}
The empirical Rademacher complexity of $\setL$ on $S$ is
\begin{equation}\label{eq:rad_definition}
\hatR_S(\setL)=\mathbb{E}_{\boldsymbol{\sigma}}\left[
\sup_{\ell_{\vphi}\in\setL}\frac{1}{N}\sum_{i=1}^N\sigma_i\ell_{\vphi}(s_i)\right],
\quad \sigma_i\overset{\mathrm{i.i.d.}}{\sim}\mathrm{Uniform}\{-1,+1\}.
\end{equation}

\begin{theorem}[Standard Rademacher generalization bound~\cite{mohri2018foundations}]\label{thm:standard_gen}
Assume that $\ell_{\vphi}(s)\in[0,B_\ell]$ for every $\ell_{\vphi}\in\setL$.
For any $\delta\in(0,1)$, with probability at least $1-\delta$ over the draw of $S\sim\mathcal{D}^N$,
\begin{equation}\label{eq:standard_gen}
\sup_{\vphi}\bigl(\Etrue(\vphi)-\hatE(\vphi)\bigr)
\le 2\hatR_S(\setL)+3B_\ell\sqrt{\frac{\log(2/\delta)}{2N}}.
\end{equation}
\end{theorem}

For the theoretical analysis, the diffusion noise is assumed to satisfy $\norm{\vepsilon}_2\le\sqrt d$ almost surely.
Since $\norm{g_{\vphi}(\vz_i)}_2\le A$, the triangle inequality gives $B_\ell=A+\sqrt d$.
Let $R_x:=\max_i\norm{\vx_i}_2$, $R_\theta:=\max_i\norm{\vtheta_i}_2$, and $R_\tau:=\max_i\norm{\vtau_{t_i}}_2$.
Then the model inputs satisfy $\norm{\vz_i}_2\le R_z$, where
$R_z:=R_x+R_\theta+R_\tau+\sqrt d$.
The map $\vy\mapsto\norm{\vy-\vepsilon_i}_2$ is $1$-Lipschitz, so the vector-contraction inequality~\cite[Corollary 4]{maurer2016vector} gives
$\hatR_S(\setL)
\le\frac{\sqrt{2}}{N}\mathbb{E}_{\sigma_{ir}}
\left[\sup_{g_{\vphi}\in\setG}\sum_{i=1}^N\sum_{r=1}^d
\sigma_{ir}g_{\vphi,r}(\vz_i)\right]$,
where $\sigma_{ir}$ are independent Rademacher variables and $g_{\vphi,r}$ is the $r$th component of $g_{\vphi}$.
The weight bound in~\eqref{eq:model_class} together with scalar contraction, the Efron--Stein inequality, and Jensen's inequality bounds the Rademacher complexity as
\begin{equation}\label{eq:base_rad}
\hatR_S(\setL)\le\sqrt{12}A L_1 W\frac{\sqrt d}{\sqrt N}R_z.
\end{equation}

For the box-constrained QP in~\eqref{eq:box_qp} with diffusion models in Sec.~\ref{sec:background}, the $j$th iterate in $\mathcal{S}^{\mathrm{solver}}_{k,i}$ becomes
$\vz_{i,j}:=\left(\sqrt{\bar\alpha_{t_i}}T_{\vtheta_i}^{n_i-j}(\vx_i^0)
+\sqrt{1-\bar\alpha_{t_i}}\vepsilon_{i,j},\vtheta_i,\vtau_{t_i}\right)$, with the training data $\mathcal{S}^{\mathrm{train}}_{k,i}:=\bigl(\vz_{i,j}\bigr)_{j=0}^k$.
We then define the discounted average loss:
\begin{equation}\label{eq:path_loss}
J_i(\vphi)=\frac{1}{k+1}\sum_{j=0}^{k}\gamma^j
\ell(T_{\vtheta_i}^{n_i-j}(\vx_i^0),\vtheta_i,\vepsilon_{i,j},t_i;\vphi),
\qquad 0<\gamma\le1.
\end{equation}
The diffusion noise $\vepsilon_{i,j}$ is associated with iterate $j$, and $t_i$ is the sampled diffusion step.
$\gamma<1$ gives smaller weights to earlier iterates.
The $k$-neighborhood path loss class is defined as
$\setJ:=\{\mathcal{S}^{\mathrm{train}}_{k,i}\mapsto J_i(\vphi):g_{\vphi}\in\setG\}$.
In the empirical and true risk definitions in~\eqref{eq:risk_definitions}, each sample $s_i$ is replaced by $\mathcal{S}^{\mathrm{train}}_{k,i}$ and each pointwise loss $\ell_{\vphi}$ is replaced by $J_i$.
The path loss is bounded by $B_J:=B_\ell$.

Let $R_{x^*}=\max_i\norm{\vx^*(\vtheta_i)}_2$, $R_\theta=\max_i\norm{\vtheta_i}_2$, and $R_\tau=\max_i\norm{\vtau_{t_i}}_2$.
Let $r_k=\max_i r_k(\vtheta_i)$ and let $\rho=\max_i\rho_i<1$ be a uniform contraction factor.
The contraction bound~\eqref{eq:k_path_contraction} and the triangle inequality give
$
\norm{T_{\vtheta_i}^{n_i-j}(\vx_i^0)}_2
\le R_{x^*}+\rho^{k-j}r_k$.
Define $R_{z,j}:=\max_i\norm{\vz_{i,j}}_2$.
The $j$th input to the diffusion model satisfies
\begin{equation}\label{eq:z_radius}
R_{z,j}\le R_0+\rho^{k-j}r_k, \quad \text{where} \quad R_0:=R_{x^*}+R_\theta+R_\tau+\sqrt d.
\end{equation}

\begin{theorem}[PGD generalization bound]\label{thm:k_bound}
Let the $k$-neighborhood paths collected in training be i.i.d.
Assume that $n_i\ge k$ for all $i$, $\norm{\vepsilon_{i,j}}_2\le\sqrt d$ for all $i$ and $j$, $r_k<\infty$, and $\rho<1$.
For any $\delta\in(0,1)$, with probability at least $1-\delta$,
\begin{align}
\sup_{\vphi}\bigl(\Etrue(\vphi)-\hatE(\vphi)\bigr)
&\le 2\hatR_{\mathcal{S}^{\mathrm{train}}_k}(\setJ)+3B_J\sqrt{\frac{\log(2/\delta)}{2N}},\label{eq:k_gen}\\
\hatR_{\mathcal{S}^{\mathrm{train}}_k}(\setJ)
&\le \sqrt{12}A L_1W\frac{\sqrt d}{\sqrt N}
\left[\frac1{k+1}\sum_{j=0}^{k}\gamma^j
\left(R_0+\rho^{k-j}r_k\right)\right].\label{eq:k_rad}
\end{align}
Here, $\mathcal{S}^{\mathrm{train}}_k:=\{\mathcal{S}^{\mathrm{train}}_{k,i}\}_{i=1}^N$ is the collection of $k$-neighborhood paths used for training.
\end{theorem}

\begin{proof}
Apply Theorem~\ref{thm:standard_gen} to the $N$ independent $k$-neighborhood path samples with loss class $\setJ$ and bound $B_J$.
By subadditivity of empirical Rademacher complexity, the complexity of the average in~\eqref{eq:path_loss} is at most the average complexity of its $k+1$ component losses.
For each $j$, we apply~\eqref{eq:base_rad}, scale by $\gamma^j$, and substitute~\eqref{eq:z_radius}.
Summing over $j=0,\ldots,k$ gives~\eqref{eq:k_rad}.
\end{proof}

The bound scales as $O(N^{-1/2})$ and improves as we increase the number of independent solver runs.
The value of $k$ enters through the average radius and discount terms, while $\rho^{k-j}r_k$ captures the larger variation of earlier iterates.
This bound can be conservative because it uses a norm-bounded NN, the worst-case radius $r_k$, and a uniform contraction factor.

\section{Examples of PGD Trajectories}\label{sec:examples}

We illustrate PGD solver iterates with different upper bounds using two instances of the two-dimensional box-constrained QP in~\eqref{eq:box_qp}.
Both instances use
$\mQ=\left[\begin{smallmatrix}4&1.5\\1.5&1\end{smallmatrix}\right]$,
$\vb=(2,1.5)^T$, and the initial guess $\vx^0=(-1.5,-1)^T$, but use different upper bounds $\vu$.
We set $\stepsize=2/(\lambda_{\max}+\lambda_{\min})=0.4$ as in~\eqref{eq:stepsize} and $\epsilon_{\mathrm{tol}}=10^{-5}$.

Figure~\ref{fig:pgd_trajectories}(a) uses $\vu=(4,4)^T$, for which the optimum lies in the interior of the feasible set.
The iterates alternate across the optimum while contracting toward it.
Figure~\ref{fig:pgd_trajectories}(b) uses $\vu=(0.5,0.8)^T$, for which the second upper bound is active at the optimum.
Projection first maps the iterate to the boundary, after which the remaining iterates contract along the boundary toward the optimum.
Thus, box constraints can change the path of solver iterates, but later iterates always lie in smaller regions around the optimum as described in Sec.~\ref{sec:solver_iterates}.

\begin{figure}[t]
\centering
\includegraphics[width=0.9\textwidth]{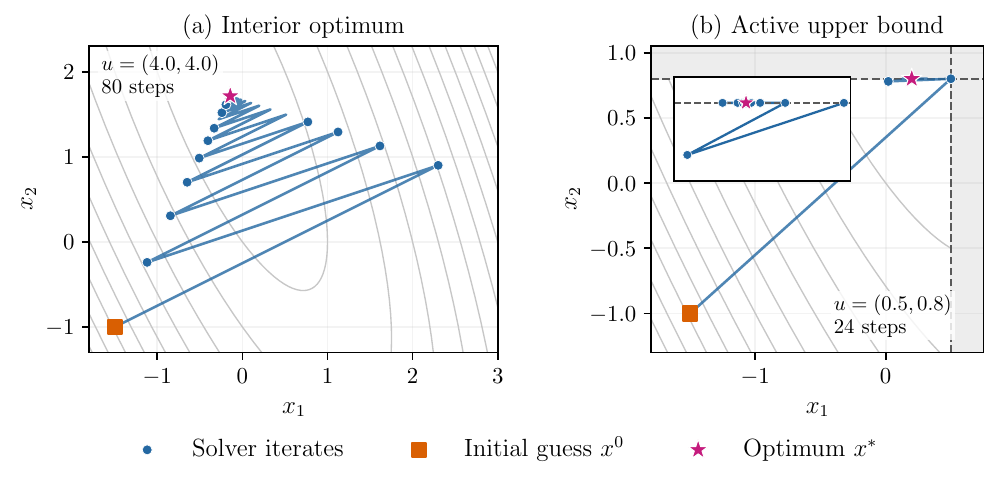}
\caption{PGD solver iterates for the same QP objective and initial guess under two upper box constraints.
The markers show the solver iterates, the dashed lines indicate the upper bounds, and the shaded region is infeasible.
The inset in (b) shows the refinement along the active boundary near the optimum.
}
\label{fig:pgd_trajectories}
\end{figure}

\section{Discussion on Learning from Solver Iterates}
\label{sec:discussion}

Training data augmented with intermediate solver iterates can reduce the data requirements of generative learning for global search.
These iterates can be viewed in two ways: as individual points along an ordered $k$-neighborhood path or as one path embedded in a higher-dimensional space.
The pointwise view can condition on the iteration index or distance to the converged solution, while the pathwise view can learn the dependence among iterates in the same solver run.

These views motivate our work GLENS (Global Search via Learning from Solver Iterates)~\cite{li2026glens}, which learns from $k$-neighborhood paths using diffusion models with two components.
The neighborhood structure model learns the local geometry around converged solutions, while the solver behavior model learns refinement directions that guide generated samples toward nearby optima during diffusion sampling.
The presented generalization bound based on Rademacher complexity provides a theoretical foundation for GLENS, the corresponding method for more data-efficient global search.

\section{Conclusion}
\label{sec:conclusion}

To reduce the data requirements of generative learning for global search, we propose the $k$-neighborhood as the iterates near final converged solutions and establish the contraction behavior of PGD for one-sided box-constrained QPs.
A generalization bound with Rademacher complexity is derived that captures how the number of solver runs, contraction factor, neighborhood radius, and discount factor affect the learning performance.
With the pointwise and pathwise views of learning solver iterates, our work builds a foundation for the data-efficient global search method GLENS~\cite{li2026glens}.
Together, these results connect solver data, numerical optimization, and generative learning within the DDDAS paradigm.
Future work will develop tighter NN bounds, directly model dependence of solver iterates, and extend the analysis to unbounded Gaussian diffusion noise.

\begin{credits}
\subsubsection{\ackname}
The authors acknowledge partial support during the course of this effort from the Princeton Laboratory for Artificial Intelligence’s [PLI/AI\textsuperscript{2}/NAM] initiative and the Air Force Office of Scientific Research under Grant \#FA9550-25-1-0100.
The authors thank Bartolomeo Stellato for a number of helpful conversations regarding the projected gradient descent algorithm during the course of the work. 
\end{credits}

\bibliographystyle{splncs04}
\bibliography{main_short.bbl}

\end{document}

%% file: macros.tex
\newcommand{\R}{\mathbb{R}}

\newcommand{\vx}{\mathbf{x}}
\newcommand{\vb}{\mathbf{b}}
\newcommand{\vy}{\mathbf{y}}
\newcommand{\vz}{\mathbf{z}}
\newcommand{\vtheta}{\boldsymbol{\theta}}
\newcommand{\vphi}{\boldsymbol{\phi}}
\newcommand{\va}{\boldsymbol{a}}
\newcommand{\vw}{\boldsymbol{w}}
\newcommand{\vu}{\boldsymbol{u}}
\newcommand{\vepsilon}{\boldsymbol{\epsilon}}
\newcommand{\vtau}{\boldsymbol{\tau}}
\newcommand{\mQ}{Q}

\newcommand{\stepsize}{\eta}
\newcommand{\setC}{\mathcal{C}}
\newcommand{\setG}{\mathcal{G}}
\newcommand{\setL}{\mathcal{L}}
\newcommand{\setJ}{\mathcal{J}}
\newcommand{\hatR}{\widehat{\mathfrak{R}}}
\newcommand{\hatE}{\widehat{\mathcal{E}}}
\newcommand{\Etrue}{\mathcal{E}}
\newcommand{\ProjC}[1]{P_{\mathcal{C}}(#1)}
\newcommand{\norm}[1]{\lVert #1 \rVert}

%% file: main_short.bbl
\begin{thebibliography}{10}
\providecommand{\url}[1]{\texttt{#1}}
\providecommand{\urlprefix}{URL }
\providecommand{\doi}[1]{https://doi.org/#1}

\bibitem{angelopoulos2023conformal}
Angelopoulos, A.N., Bates, S.: Conformal prediction: A gentle introduction. Foundations and Trends in Machine Learning  \textbf{16}(4),  494--591 (2023)

\bibitem{beeson2024global}
Beeson, R., Li, A., Sinha, A.: Global search of optimal spacecraft trajectories using amortization and deep generative models. arXiv preprint arXiv:2412.20023  (2024)

\bibitem{Graebner:2025.jas.72.6}
Graebner, J., Beeson, R.: Global search for optimal low thrust spacecraft trajectories using diffusion models and the indirect method. The Journal of the Astronautical Sciences  \textbf{72}(6), ~62 (2025). \doi{10.1007/s40295-025-00535-1}

\bibitem{Graebner:2026.arxiv.jass}
Graebner, J., Beeson, R.: Transfer learning of multiobjective indirect low-thrust trajectories using diffusion models and markov chain monte carlo (5 2026). \doi{10.48550/arXiv.2605.09125}

\bibitem{ho2020denoising}
Ho, J., Jain, A., Abbeel, P.: Denoising diffusion probabilistic models. Advances in neural information processing systems  \textbf{33},  6840--6851 (2020)

\bibitem{li2025aligning}
Li, A., Beeson, R.: Aligning diffusion model with problem constraints for trajectory optimization. In: Blasch, E., Darema, F., Metaxas, D. (eds.) Handbook of Dynamic Data-Driven Applications Systems, vol.~4. Springer, Cham (2025), preprint available as arXiv:2504.00342

\bibitem{li2025diffusolve}
Li, A., Ding, Z., Dieng, A.B., Beeson, R.: {DiffuSolve}: Diffusion-based solver for non-convex trajectory optimization. In: Ozay, N., Balzano, L., Panagou, D., Abate, A. (eds.) Proceedings of the 7th Annual Learning for Dynamics \& Control Conference. Proceedings of Machine Learning Research, vol.~283, pp. 45--58. PMLR (4--6 June 2025)

\bibitem{li2024constraint}
Li, A., Ding, Z., Dieng, A.B., Beeson, R.: Constraint-aware diffusion models for trajectory optimization. In: Blasch, E., Darema, F., Metaxas, D. (eds.) Dynamic Data Driven Applications Systems. pp. 308--316. Springer Nature Switzerland, Cham (2026)

\bibitem{li2023amortized}
Li, A., Sinha, A., Beeson, R.: Amortized global search for efficient preliminary trajectory design with deep generative models. In: AAS/AIAA Astrodynamics Specialist Conference. Big Sky, MT (August 2023), paper AAS 23-352

\bibitem{li2026glens}
Li, A., Stellato, B., Beeson, R.: Glens: Global search via learning from solver iterates with diffusion models. arXiv preprint arXiv:2606.00366  (2026)

\bibitem{maurer2016vector}
Maurer, A.: A vector-contraction inequality for rademacher complexities. In: International Conference on Algorithmic Learning Theory. pp. 3--17. Springer (2016)

\bibitem{mcallester1998some}
McAllester, D.A.: Some pac-bayesian theorems. In: Proceedings of the eleventh annual conference on Computational learning theory. pp. 230--234 (1998)

\bibitem{mohri2018foundations}
Mohri, M., Rostamizadeh, A., Talwalkar, A.: Foundations of Machine Learning. MIT Press, Cambridge, MA, 2 edn. (2018)

\end{thebibliography}
